\documentclass[11pt]{article}

\usepackage{subcaption}
\usepackage{amssymb,amsmath,amsthm}
\usepackage{color}
\usepackage{graphicx}
  \DeclareGraphicsExtensions{.jpeg,.png,.jpg,.eps,.pdf}
\usepackage{tikz}
\usepackage{epstopdf}
\usepackage{algorithm}
\usepackage{subcaption}
\usepackage{float}
\usepackage{multicol}
\usepackage[noend]{algpseudocode}
\usepackage{comment}

\def\oneb{{\bf 1}}
\def\1{{\bf 1}}

\def\i{{\bf i}}

\def\A{\mathcal{A}}

\def\C{\mathcal{C}}

\def\E{\mathcal{E}}

\def\I{\mathcal{I}}

\def\M{\mathcal{M}}
\def\N{\mathcal{N}}

\def\S{\mathcal{S}}
\def\T{\mathcal{T}}

\def\V{\mathcal{V}}

\def\Fbb{{\mathbb F}}

\def\R{{\mathbb R}}

\def\Zbb{{\mathbb Z}}

\def\al{\alpha}
\def\d{\delta}
\def\D{\Delta}
\def\e{\epsilon}
\def\g{\gamma}

\def\l{\lambda}

\def\OM{\Omega}

\def\s{\sigma}

\def\th{\theta}

\def\ib{\bar{i}}

\def\jb{\bar{j}}

\def\Xb{\bar{X}}

\def\Yb{\bar{Y}}

\def\Zb{\bar{Z}}

\def\nai{n \ap \infty}

\def\ap{\rightarrow}

\def\es{\emptyset}
\def\seq{\subseteq}

\def\bi{\{0,1\}}

\def\bp{\{-1,1\}}

\def\fa{\; \forall}

\def\drc{(d_r,d_c)}

\def\st{\mbox{ s.t. }}

\def\nm{\Vert}

\renewcommand{\and}{\mbox{$\wedge$}}

\newcommand{\bc}{\begin{center}}
\newcommand{\ec}{\end{center}}
\newcommand{\be}{\begin{equation}}
\newcommand{\ee}{\end{equation}}
\newcommand{\bd}{\begin{displaymath}}
\newcommand{\ed}{\end{displaymath}}
\newcommand{\ba}{\begin{array}}
\newcommand{\ea}{\end{array}}
\newcommand{\ben}{\begin{enumerate}}
\newcommand{\een}{\end{enumerate}}
\newcommand{\bit}{\begin{itemize}}
\newcommand{\eit}{\end{itemize}}
\newcommand{\beq}{\begin{eqnarray}}
\newcommand{\eeq}{\end{eqnarray}}
\newcommand{\btab}{\begin{tabular}}
\newcommand{\etab}{\end{tabular}}
\newcommand{\bfig}{\begin{figure}}
\newcommand{\efig}{\end{figure}}
\newcommand{\btp}{\begin{tikzpicture}}
\newcommand{\etp}{\end{tikzpicture}}
\newcommand{\bcm}{}

\newcommand{\argmin}{\operatornamewithlimits{argmin}}

\renewcommand{\mod}{~{\rm mod}~}

\newcommand{\nmN}[1]{ \nm #1 \nm_N }

\newcommand{\nmS}[1]{ \nm #1 \nm_S }

\newcommand{\IP}[2]{ \langle #1 , #2 \rangle }

\newcommand{\Leg}[2]{ \left( \frac{#1}{#2} \right) }

\newcommand{\rk}{{\rm{rank}}}
\newcommand{\supp}{{\rm{supp}}}

\def\Fq{\Fbb_q}
\def\Fqs{\Fbb_q^*}
\def\PGq{PGL(2,\Fq)}
\def\PSq{PSL(2,\Fq)}
\def\PScq{PSL^c(2,\Fq)}

\def\Rno{\R^{n_r \times n_c}}

\def\Xh{\hat{X}}

\def\nmsl1{\nm_{{\rm SL1}}}

\definecolor{verm}{rgb}{0.6,0.2,0.2}
\definecolor{purp}{rgb}{0.3,0.1,0.6}
\definecolor{purple}{rgb}{0.4,0.0,0.6}
\definecolor{bggreen}{rgb}{0.1,0.3,0.1}
\definecolor{dgreen}{rgb}{0.1,0.6,0.1}
\definecolor{black}{rgb}{0.0,0.0,0.0}
\definecolor{crim}{rgb}{0.3,0.1,0.1}
\definecolor{dred}{rgb}{0.5,0.1,0.1}

\definecolor{Blue}{cmyk}{0.65,0.13,0,0}
\definecolor{Black}{cmyk}{0,0,0,1}
\definecolor{Red}{cmyk}{0,1,1,0}
\definecolor{Green}{cmyk}{1,0,1,0}
\definecolor{Orange}{cmyk}{0,0.61,0.87,0.1}
\definecolor{Fuchsia}{cmyk}{0.47,0.91,0,0.08}
\definecolor{PineGreen}{cmyk}{0.92,0,0.59,0.25}

\newtheorem{corollary}{Corollary}
\newtheorem{definition}{Definition}

\newtheorem{lemma}{Lemma}
\newtheorem{theorem}{Theorem}

\begin{document}

\title{New and Explicit Constructions of Unbalanced Ramanujan Bipartite Graphs}

\author{Shantanu Prasad Burnwal and Kaneenika Sinha and M.\ Vidyasagar
\thanks{SPB is pursuing his Ph.D.\ in Electrical Engineering at
the Indian Institute of Technology Hyderabad; 
email: ee16resch11019@iith.ac.in.
KS is with the Department of Mathematics, Indian Institute of Science,
Education, and Research (IISER), Pune; email: kaneenika@iiserpune.ac.in.
MV is with the Department of Electrical Engineering at
the Indian Institute of Technology Hyderabad; email: M.Vidyasagar@iith.ac.in.
The first author was supported by a Ministry of Human Reseource
Development (MHRD), Government of India.
The second author was supported by a MATRICS grant from the
Science and Engineering Research Board (SERB), Department of Science
and Technology, Government of India.
The third author was supported by a SERB National Science Chair,
Government of India.}}



\date{\today}

\maketitle

\begin{abstract}
The objectives of this article are three-fold.  Firstly, we present for the first time explicit constructions of an infinite family of \textit{unbalanced} Ramanujan bigraphs.  Secondly, we revisit some of the known methods for constructing Ramanujan graphs and discuss the computational work required in actually implementing the various construction methods.  The third goal of this article is to address the following question: can we construct a bipartite Ramanujan graph with specified degrees, but with the restriction that the edge set of this graph must be distinct from a given set of ``prohibited" edges?  We provide an affirmative answer in many cases, as long as the set of prohibited edges is not too large.
\end{abstract}



\section{Introduction}\label{sec:Intro}

A fundamental theme in graph theory is the study of the spectral gap
of a regular (undirected) graph, that is, the difference between the
two largest eigenvalues of the adjacency matrix of such a graph.
Analogously, the spectral gap of a bipartite, biregular graph is the
difference between the two largest singular values of its biadjacency
matrix (please see Section \ref{sec:Review} for detailed definitions). 
Ramanujan graphs (respectively Ramanujan bigraphs) are graphs with an 
optimal spectral gap. 
Explicit constructions of such graphs have multifaceted applications in
areas such as computer science, coding theory and compressed sensing. 
In particular, in \cite{Shantanu-TSP20} the first and third authors show
that Ramanujan graphs can provide the first \textit{deterministic} solution
to the so-called matrix completion problem.
Prior to the publication of \cite{Shantanu-TSP20}, the matrix completion
problem had only a probabilistic solution.
The explicit construction of Ramanujan graphs has been classically well studied.
Some explicit methods are known, for example,  by the work of Lubotzky,
Phillips and Sarnak \cite{Lubotzky-et-al88}, Li \cite{Li92}, Morgenstern
\cite{Morgenstern94}, Gunnells \cite{Gunnells05}, Bibak, Kapron and
Srinivasan \cite{Bibak-et-al16}. 
These methods draw from concepts in linear algebra, number theory,
representation theory and the theory of automorphic forms. 
In contrast, however, no explicit methods for constructing 
\textit{unbalanced} Ramanujan bigraphs are known. 
There are a couple of abstract constructions
in \cite{Ballantine-et-al15,Evra-Parz18}, but these are not
explicit.\footnote{There is however a paper under
preparation by C.\ Ballantine, S.\ Evra, B.\ Feigon, K.\ Maurischat,
and O.\ Parzanchevski that will present an explicit constrution.}

This article has three goals.
\begin{enumerate}
	\item {\bf [Explicit construction of unbalanced Ramanujan bigraphs]} First and foremost, we present for the first time explicit constructions of an infinite family of \textit{unbalanced} Ramanujan bigraphs.  Our construction, presented in Section \ref{sec:Rama-B}, is based on ``array code" matrices from LDPC (low density parity check) coding theory.  Apart from being the first explicit constructions, they also have an important computational feature: the biadjacency matrices are obtained \textit{immediately} upon specifying two parameters, a prime number $q$ and an integer $l \geq 2$.  

\item {\bf [Comparison of computational aspects of known constructions]} The second goal of this article is to revisit some of the earlier known constructions of Ramanujan graphs and (balanced) bigraphs and compare the amount of work involved in  constructing the various classes of graphs and obtaining their adjacency or biadjacency matrices.  We also focus on the two earlier-known constructions of Ramanujan bipartite graphs due to Lubotzky-Phillips-Sarnak \cite{Lubotzky-et-al88} and Gunnells \cite{Gunnells05} and show that each can be converted into a non-bipartite graph.  Note that every graph can be associated with a bipartite graph, but the converse is not true in general.  Also, the research community prefers non-bipartite graphs over bipartite graphs.  Thus, our proof that the LPS and Gunnells constructions can be converted to non-bipartite graphs is of some interest.  All of this is addressed in Section \ref{sec:Rama-G}.

\item {\bf [Construction of Ramanujan graphs with prohibited edges]}
The third goal of this article is to address the following question:
can we construct a bipartite Ramanujan graph with specified degrees,
but with the restriction that the edge set of this graph must be distinct
from a given set of ``prohibited'' edges? 
The approach that we follow to answer this question is to start with an
existing Ramanujan bigraph, and then to perturb its edge set
so as to eliminate the prohibited edges and replace them by other edges
that are non prohibited.
This procedure retains the biregularity of the graph.
We then show that our replacement procedure also retains the Ramanujan
nature of the bigraph, provided the gap between the second largest
singular value of the biadjacency matrix and the ``Ramanujan bound''
is larger than twice the maximum number of prohibited edges at each vertex.
These questions are studied in Section \ref{sec:Pert}.  

\end{enumerate}
\subsection{Organization of paper}
\label{sec: Organization}
This article is organized as follows.  In Section \ref{sec:Review}, we present a brief review of Ramanujan graphs and bigraphs.  

In Section \ref{sec:Rama-B}, we address the first goal of this article,
and present the first explicit construction of an infinite
 family of unbalanced Ramanujan bigraphs.  

In Sections \ref{sec:Rama-G} (Sections \ref{ssec:31} - \ref{ssec:34}), we address the second goal of this article. 
We review many of the known methods for constructing Ramanujan graphs, based on the original publications. 

In Section \ref{sec:Further}, we shed further light on the two constructions of Ramanujan bipartite graphs due to Lubotzky-Phillips-Sarnak \cite{Lubotzky-et-al88} and Gunnells \cite{Gunnells05} and show how each can be converted into a non-bipartite graph. 
 
In Section \ref{sec:Asp}, we analyze the computational effort required in actually implementing the various construction methods reviewed in Section \ref{sec:Rama-G}.  A pertinent issue that is addressed here is whether one can give a polynomial-time algorithm for implementing the known constructions.
 In Section \ref{sec:Pert} (Sections \ref{ssec:71} and \ref{sec:72}), we accomplish the third goal of this article, namely to construct a bipartite Ramanujan graph with specified degrees, but with the restriction that the edge set of this graph must be disjoint from a given set of ``prohibited" edges.
In Section \ref{ssec:73}, we analyze our new construction of
Ramanujan bigraphs from Section \ref{sec:Rama-B} as well as the previously
known constructions of Ramanujan graphs and bigraphs,
in terms of 
how many edges can be relocated while retaining the
Ramanujan property.

\section{Review of Ramanujan Graphs and Bigraphs}\label{sec:Review}


In this subsection we review the basics of Ramanujan graphs and
Ramanujan bigraphs.
Further details about Ramanujan graphs can be found in
\cite{Ram-Murty03,DSV03}.

Recall that a \textbf{graph} consists of a vertex set $\V$ and an edge set
$\E \seq \V \times \V$.
If $(v_i,v_j) \in \E$ implies that $(v_j,v_i) \in \E$, then the graph is
said to be \textbf{undirected}.
A graph is said to be \textbf{bipartite} if $\V$ can be partitioned into
two sets $\V_r, \V_c$ such that $\E \cap (\V_r \times \V_r) = \es$,
$\E \cap (\V_c \times \V_c) = \es$.
Thus, in a bipartite graph, all edges connect one vertex in $\V_r$ with
another vertex in $\V_c$.
A bipartite graph is said to be \textbf{balanced} if $|\V_r| = |\V_c|$,
and \textbf{unbalanced} otherwise.

A graph is said to be \textbf{$d$-regular} if every vertex has the same
degree $d$.
A bipartite graph is said to be \textbf{$\drc$-biregular}
if every vertex in $\V_r$
has degree $d_r$, and every vertex in $\V_c$ has degree $d_c$.
Clearly this implies that $d_c |\V_c| = d_r |\V_r|$.

Suppose $(\V,\E)$ is a graph.
Then its adjacency matrix $A \in \bi^{|\V| \times |\V|}$ is defined by
setting $A_{ij} = 1$ if there is an edge $(v_i,v_j) \in \E$, and
$A_{ij} = 0$ otherwise.
In an undirected graph (which are the only kind we deal with in the paper),
$A$ is symmetric and therefore has only real eigenvalues.
If the graph is $d$-regular, then $d$ is an eigenvalue of $A$ and is also
its spectral radius.
The multiplicity of $d$ as an eigenvalue of $A$ equals the number of connected
components of the graph.
Thus the graph is connected if and only if $d$ is a simple eigenvalue of $A$.
The graph is bipartite if and only if $-d$ is an eigenvalue of $A$.
If the graph is bipartite, then its adjacency matrix $A$ looks like
\bd
A = \left[ \ba{cc} 0 & B \\ B^\top & 0 \ea \right] ,
\ed
where $B \in \bi^{|\V_r| \times |\V_c|}$ is called the
\textbf{biadjacency matrix}.
The eigenvalues of $A$ equal $\pm \s_1 , \ldots , \pm \s_l$ together
with a suitable number of zeros, where $l = \min \{ |\V_r| , |\V_c| \}$,
and $\s_1 , \ldots , \s_l$ are the singular values of $B$.
In particular, in a $\drc$-biregular graph, $\sqrt{d_r d_c}$ is the
largest singular value of $B$.
These and other elementary facts about graphs can be found in
\cite{Ram-Murty03}.

\begin{definition}\label{def:Rama-G}
A $d$-regular graph is said to be a \textbf{Ramanujan graph} if the
second largest eigenvalue by magnitude of its adjacency matrix,
call it $\l_2$, satisfies
\be\label{eq:21}
| \l_2 | \leq 2 \sqrt{d-1} .
\ee
A $d$-regular bipartite graph\footnote{Note that such a bipartite graph
must perforce be balanced with $|\V_r| = |\V_c|$ and $d_r = d_c = d$.}
is said to be a \textbf{bipartite Ramanujan graph} if the second largest
singular value of its biadjacency matrix, call it $\s_2$, satisfies
\be\label{eq:22}
\s_2 \leq 2 \sqrt{d-1} .
\ee
\end{definition}

Note the distinction being made between the two cases.
If a graph is $d$-regular and bipartite, then it cannot be a Ramanujan graph,
because in that case $\l_2 = -d$, which violates \eqref{eq:21}.
On the other hand, if it satisfies \eqref{eq:22}, then it is called
a bipartite Ramanujan graph.
Observe too that not all authors make this distinction.

\begin{definition}\label{def:Rama-B}
A $\drc$-biregular bipartite graph is said to be a \textbf{Ramanujan bigraph}
if the second largest
singular value of its biadjacency matrix, call it $\s_2$, satisfies
\be\label{eq:23}
\s_2 \leq \sqrt{d_r-1} + \sqrt{d_c-1} .
\ee
\end{definition}

It is easy to see that Definition \ref{def:Rama-B} contains the second case of
Definition \ref{def:Rama-G} as a special case when $d_r = d_c = d$.
A Ramanujan bigraph with $d_r \neq d_c$ is called an \textbf{unbalanced
Ramanujan bigraph}.

The rationale behind the bounds in these definitions is given the following
results.
In the interests of brevity, the results are paraphrased and the reader
should consult the original sources for precise statements.

\begin{theorem}\label{thm:21}
(Alon-Boppana bound; see \cite{Alon86}.)
Fix $d$ and let $\nai$ in a $d$-regular graph with $n$ vertices.
Then
\be\label{eq:24}
\liminf_{\nai} | \l_2 | \geq 2 \sqrt{d-1} .
\ee
\end{theorem}

\begin{theorem}\label{thm:22}
(See \cite{Feng-Li96}.)
Fix $d_r,d_c$ and let $n_r, n_c$ approach infinity subject to $d_r n_r = d_c n_c$.
Then
\be\label{eq:25}
\liminf_{n_r \ap \infty , n_c \ap \infty} \s_2 \geq \sqrt{d_r-1} + \sqrt{d_c-1} .
\ee
\end{theorem}

Given that a $d$-regular graph has $d$ as its largest eigenvalue $\l_1$,
a Ramanujan graph is one for which the ratio $\l_2/\l_1$ is as small as
possible, in view of the Alon-Boppana bound of Theorem \ref{thm:21}.
Similarly, given that a $\drc$-regular bipartite graph has
$\s_1 = \sqrt{d_r d_c}$, a Ramanujan bigraph is one for which the ratio
$\s_2/\s_1$ is as small as possible, in view of Theorem \ref{thm:22}.

In a certain sense, Ramanjuan graphs and Ramanujan bigraphs are pervasive.
To be precise, if $d$ is kept fixed and $\nai$, then the fraction of
$d$-regular, $n$-vertex graphs that satisfy the Ramanujan property
approaches one; see \cite{Friedman03,Friedman08}.
Similarly, if $d_r,d_c$ are kept fixed and $n_r, n_c \ap \infty$
(subject of course to the condition that $d_r n_r = d_c n_c$, then
the fraction of $\drc$-biregular graphs that are Ramanujan bigraphs
approaches one; see \cite{Brito-et-al-arxiv18}.
However, despite their prevalence, there are relatively few \textit{explicit}
methods for constructing Ramanujan graphs.
Many of the currently known techniques are reprised in Section \ref{sec:Rama-G}.
\section{Two New Families of Unbalanced Ramanujan Bigraphs}\label{sec:Rama-B}

At the moment, there is not a single \textit{explicitly constructible}
family of unbalanced Ramanujan bigraph, that is, an unbalanced biregular
bipartite graph that satisfies the inequality \eqref{eq:23}.
The objective of this section is to present what the authors believe is
the first such explicit construction.

According to a recent result \cite{Brito-et-al-arxiv18}, randomly generated
$\drc$-biregular bipartite graphs with $(n_r,n_c)$ vertices 
satisfy the Ramanujan property with probability approaching one as
$n_r,n_c$ simultaneously approach infinity (of course, while satisfying
the constraint that $d_r n_r = d_c n_c$).
This result generalizes an earlier result due to \cite{Friedman03,Friedman08}
which states that randomly generated $d$-regular, $n$-vertex graph
satisfies the Ramanujan property as $\nai$ with probability approaching one.
A ``road map'' for constructing Ramanujan bigraphs is given in
\cite{Ball-Ciu11}, and 
some abstract constructions of Ramanujan bigraphs are given in
\cite{Ballantine-et-al15,Evra-Parz18}.
These bigraphs have degrees $(p+1,p^3+1)$ for various values of $p$,
such as $p \equiv 5 \mod 12$, $p \equiv 11 \mod 12$ \cite{Ballantine-et-al15},
and $p \equiv 3 \mod 4$ \cite{Evra-Parz18}.\footnote{The authors thank
Prof.\ Cristina Ballantine for aiding us in interpreting these papers.}
For each suitable choice of $p$, these papers lead to an infinite family
of Ramanujan bigraphs.
At present, these constructions are not explicit in terms of resulting in a
biadjacency matrix of $0$s and $1$s.
There is a paper under preparation by the authors of these
papers to make these
constructions explicit.

In contrast, in this section, we state and prove two such explicit
constructions, namely $(lq,q^2)$-biregular graphs where $q$ is any prime and $l$
is any integer that satisfies $2 \leq l \leq q$,
and $(q^2,lq)$-biregular graphs where $q$ is any prime and $l$ is any
integer greater than $q$.
Thus we can constuct Ramanujan bigraphs for a broader range of degree-pairs
compared to \cite{Ballantine-et-al15,Evra-Parz18}.
In particular, for $l = q$ we generate a new class of
Ramanujan graphs.
However, for a given pair of integers $l,q$, we can construct only one
Ramanujan bigraph.

Our construction is based
on so-called ``array code'' matrices from LDPC (low density parity check)
coding theory, first introduced in \cite{Fan00,Yang-Hell-TIT03}.
Suppose $q$ is a prime number, and let
$P\in \bi^{q \times q}$ be a cyclic shift permutation matrix on $q$ numbers.
Then the entries of $P$ can be expressed as
\bd
P_{ij} = \left\{ \ba{ll} 1, & j = i-1 \mod q \\ 0, & \mbox{otherwise} .
\ea \right.
\ed
Now let $q$ be a prime number, and define
\be\label{eq:51}
B(q,l)
= \left[\ba{ccccc}
I_q & I_q & I_q  & \cdots & I_q\\
I_q &  P  & P^2  & \cdots & P^{(l-1)}\\
I_q & P^2 & P^4  & \cdots & P^{2(l-1)}\\
I_q & P^3 & P^6  & \cdots & P^{3(l-1)}\\
\vdots & \vdots  & \vdots & \vdots & \vdots\\
I_q & P^{(q-1)}  & P^{2(q-1)} & \cdots & P^{(l-1)(q-1)}\\
\ea \right]
\ee
where $P^j$ represents $P$ raised to the power $j$.
Now $B = B(q,l)$ is binary with $q^2$ rows and $lq$ columns,
row degree of $l$ and column degree of $q$.
If $l < q$, we study the matrix $B^\top$, whereas if $l \geq q$,
we study $B$.
In either case, the largest singular value of $B$ is $\sqrt{ql}$.
The fact that these bipartite graphs satisfy the Ramanujan property
is now established.

\begin{theorem}\label{thm:51}
\ben
\item Suppose $2 \leq l \leq q$.
Then the matrix $B^\top$ has a simple singular value of $\sqrt{ql}$,
$l(q-1)$ singular values of $\sqrt{q}$, and $l-1$ singular values of zero.
Therefore $B^\top$ represents a Ramanujan bigraph.
\item Suppose $l \geq q$.
The matrix $B$ has a simple singular value of $\sqrt{ql}$.
Now two subcases need to be considered:
\ben
\item When $l \mod q = 0$, in addition $B$ has
$(q-1)q$ singular values of $\sqrt{l}$
and $q-1$ singular values of $0$.
\item When $l \mod  q \neq 0$, let $k=l \mod  q$.
Then $B$ has, in addition,
$(q-1)k$ singular values of $\sqrt{l+q-k}$,
$(q-1)(q-k)$ singular values of $\sqrt{l-k}$, and
 $q-1$ singular values of $0$.
\een
Therefore, whenever $l\geq q$, $B(q,l)$ represents a Ramanujan bigraph.
\een
\end{theorem}

\begin{corollary}\label{coro:511}
For every prime number $q$, the matrix $B(q,q)$ defined in \eqref{eq:51}
is square, and is the adjacency matrix of a Ramanujan graph
with $d = q$ and $n = q^2$.
In this case $\l_1 = q$, and $\l_2 = \sqrt{q}$ with a multiplicity of
$q(q-1)$. The remaining $q-1$ eigenvalues of $B(q,q)$ are zero.
\end{corollary}

\begin{proof}
Let $B$ be a shorthand for $B(q,l)$.
Note that $P$ is a cyclic shift permutation; therefore $P^\top = P^{-1}$.
The proof consists of computing $BB^\top$, $B^\top B$ and determining its eigenvalues.
Throughout we make use of the fact that $P^\top = P^{-1}$.

We begin with the case $l \leq q$.
Use block-partition notation to divide $B B^\top$ into $l$ blocks of
size $q \times q$.
Then
\begin{eqnarray*}
(B B^\top)_{ij} & = & \sum_{s=1}^q P^{(i-1)(s-1)} (P^\top)^{(s-1)(j-1)} \\
& = & \sum_{s=1}^q P^{(i-j)(s-1)} = \sum_{s=0}^{q-1} P^{(i-j)s} .
\end{eqnarray*}
It readily follows that
\bd
(B B^\top)_{ii} = q I_q , i = 1 , \ldots , q .
\ed
Now observe that, for any nonzero integer $k$, the set of numbers $ks \mod q$
as $s$ varies over $\{ 0 , \ldots , q-1 \}$ equals $\{ 0 , \ldots , q-1 \}$.
(This is where we use the fact that $q$ is a prime number.)
Therefore, whenever $i \neq j$, we have that
\bd
(B B^\top)_{ij} = \sum_{s=0}^{q-1} P^s = \oneb_{q \times q} ,
\ed
where $\oneb_{q \times q}$ denotes the $q \times q$ matrix whose entries are
all equal to one.
Observe that the largest eigenvalue of $B B^\top$ is $ql$, with
normalized eigenvector $(1/\sqrt{ql}) \oneb_{ql}$.
Therefore if we define $M_l = B B^\top - \oneb_{ql \times ql}$
and partition it commensurately with $B$, we see that the off-diagonal blocks
of $M_l$ are all equal to zero, while the diagonal blocks are all identical
and equal to $q I_q - \oneb_{q \times q}$.
This is the Laplacian matrix of a fully connected graph with $q$ vertices,
and thus has $q-1$ eigenvalues of $q$ and one eigenvalue of $0$.
Therefore $M_l = B B^\top - \oneb_{ql \times ql}$ has $l(q-1)$ eigenvalues
of $q$ and $l$ eigenvalues of $0$.
Moreover, $\oneb_{ql}$ is an eigenvector of $M$ corresponding to the
eigenvalue zero.
Therefore $B B^\top = M_l + \oneb_{ql} \oneb_{ql}^\top$
has a single eigenvalue of $ql$,
$l(q-1)$ eigenvalues of $q$, and $l-1$ eigenvalues of $0$.
This is equivalent to the claim about singular values of $B^\top$.

Now we study the case where $l \geq q$.
Let $M_q \in \bi^{q^2 \times q^2}$ denote the matrix in the previous case
with $l = q$.
This matrix can be block-partitioned into $q \times q$ blocks, with
\bd
(M_q)_{ij} = \left\{ \ba{ll} q I_q - \oneb_{q \times q} , &
\mbox{if } j-i \equiv 0 \mod q , \\
0, & \mbox{otherwise} .
\ea \right.
\ed
Now consider the subcase that $l \equiv 0 \mod q$.
Then $B^\top B$ consists of $l/q$ repetitions of $M_q$ on each block row
and block column of size $q^2 \times q^2$.
Therefore each such row and column block gives $(q-1)q$ eigenvalues of magnitude
$l$ and rest of the eigenvalues of will be zero.
Next, if $l \mod q =: k \neq 0$, then $B^\top B$ consists of
$(l-k)/q$ repititions of $M_q$ on each block row and column of size $q^2 \times q^2$.
In addition it contains first $k$ column blocks of $M_q$ concatenated
$l$ times column-wise as the last column blocks.
It contains first $k$ row blocks of $M_q$ concatenated $l-k$ times row-wise as
the last block of rows.
The extra $k$ rows and colums of $B^\top B$ give
$(q-1)k$ eigenvalues of magnitude $l+(q-k)$.
Another set of row and column blocks give $(q-1)(q-k)$ eigenvalues of magnitude $l-k$.
The remaining eigenvalues are $0$.
\end{proof}

Note that, when $l = q$, the construction in \eqref{eq:51} leads to a new class
of Ramanujan graphs of degree $q$ and $q^2$ vertices.
In our terminology, it is a ``one-parameter'' family.
So far as we are able to determine, this family is new
and is not contained any other known family.


\section{Some Constructions of Ramanujan Graphs}\label{sec:Rama-G}

At present there are not too many methods for \textit{explicitly}
constructing Ramanujan graphs.
In this section we reprise most of the known methods.
Note that the authors have written \texttt{Matlab} codes for all of the
constructions in this section, except the Winnie Li construction
in Section \ref{ssec:33}; these codes are available upon request.

The available construction methods can be divided into two categories,
which for want of a better terminology we call ``one-parameter''
and ``two-parameter'' constructions.
One-parameter constructions are those for which, once the degree $d$ 
of the graph is
specified, the number $n$ of vertices is also fixed by the procedure.
In contrast, two-parameter constructions are those in which it is
possible to specify $d$ and $n$ independently.
The methods of Lubotzky-Phillips-Sarnak (LPS) and of Gunnells are
two-parameter, while those of Winnie Li and Bibak et al.\ are one-parameter.
Of course, not all combinations of $d$ and $n$ are permissible.

In this connection it is worth mentioning the results of 
\cite{MSS-STOC13a,MSS-Annals15a,MSS-FOCS15}, which show that
there exist \textit{bipartite} Ramanujan graphs of all degrees and all sizes.
However these results
do not imply the existence of Ramanujan graphs of all sizes and degrees.
Moreover, the ideas in \cite{MSS-STOC13a,MSS-Annals15a,MSS-FOCS15}
do not lead to an explicit construction.
There is a preprint \cite{Cohen16} that claims to give a polynomial time
algorithm for implementing the construction of
\cite{MSS-STOC13a,MSS-Annals15a}.
However, no code for the claimed implementation is available.
Some of the methods discussed here lead to \textit{bipartite} Ramanujan graphs.
It is presumably of interest to show that these constructions in fact
lead to non-bipartite Ramanujan graphs.
This is done below.

All but one of the constructions described below are Cayley graphs.
So we begin by describing that concept.
Suppose $G$ is a group, and that $S \seq G$ is ``symmetric'' in that
$a \in S$ implies that $a^{-1} \in S$.
Then the Cayley graph $\C(G,S)$ has the elements of $G$ as the vertex set,
and the edge set is of the form $(x,xa), x \in G, a \in S$.
Due to the symmetry of $S$, the graph is undirected even if $G$ is
noncommutative.

\subsection{Lubotzky-Phillips-Sarnak Construction}\label{ssec:31}

The Lubotzky-Phillips-Sarnak (referred to as LPS hereafter) construction
\cite{Lubotzky-et-al88}
makes use of two unequal primes $p,q$, each of which is $\equiv 1 \mod 4$.
%
As is customary, let $\Fq$ denote the finite field with $q$ elements,
and let $\Fq^*$ denote the set of nonzero elements in $\Fq$.
The general linear group $GL(2,\Fq)$ consists of all $2 \times 2$
matrices with elements in $\Fq$ whose determinant is nonzero.
If we define an equivalence relation $\sim$ on $GL(2,\Fq)$ via
$A \sim B$ whenever $A = \al B$ for some $\al \in \Fq^*$, then the
resulting set of equivalence classes $GL(2,\Fq) / \sim$ is 
the \textbf{projective general linear group} $\PGq$.
Next, it is shown in \cite{Lubotzky-et-al88} that there are exactly $p+1$
solutions of the equation
\be\label{eq:31}
p = a_0^2 + a_1^2 + a_2^2 + a_3^2 ,
\ee
where $a_0$ is odd and positive, and $a_1, a_2, a_3$ are even (positive or
negative).
Choose an integer $\i$ such that $\i^2 \equiv -1 \mod q$.
Thus $\i$ is a proxy for $\sqrt{-1}$ in the field $\Fq$.
Such an integer always exists. 

The LPS construction
is a Cayley graph where the group $G$ is the projective general linear
group $\PGq$.
The generator set $S$ consists of the $p+1$ matrices
\be\label{eq:32}
M_j = \left[ \ba{cc} a_{0j} + \i a_{1j} & a_{2j} + \i a_{3j} \\
-a_{2j} + \i a_{3j} & a_{0j} - \i a_{1j} \ea \right] \mod q ,
\ee
as $(a_{0j}, a_{1j}, a_{2j}, a_{3j})$ range over all solutions of \eqref{eq:31}.
Note that each matrix $M_j$ has determinant $p \mod q$.
It is clear that the LPS graph is $(p+1)$-regular, and the number of vertices
equals the cardinality of $\PGq$, which is $q(q^2-1)$.

To proceed further, it is necessary to introduce the \textbf{Legendre
symbol}.
If $q$ is an odd prime and $a$ is not a multiple of $q$, define
\bd
\Leg{a}{q} = \left\{ \ba{ll}
1 & \mbox{if } \exists x \in \Zbb \st x^2 \equiv a \mod q , \\
-1 & \mbox{otherwise} \ea \right.
\ed
Partition $[q-1] := \{ 1 , \cdots , q-1 \}$
into two subsets, according to
\be\label{eq:33}
S_{1,q} := \left\{ a \in [q-1] : \Leg{a}{q} = 1 \right\} ,
\ee
\be\label{eq:34}
S_{-1,q} := \left\{ a \in [q-1] : \Leg{a}{q} = -1 \right\} .
\ee
Then it can be shown that each set $S_{1,q}$ and $S_{-1,q}$ consists of
$(q-1)/2$ elements of $[q-1]$.
One of the many useful properties of the Legendre symbol is that,
for integers $a,b \in \Zbb$, neither of which is a multiple of $q$, we have
\bd
\left( \frac{ab}{q} \right) = \left( \frac{a}{q} \right)
\left( \frac{b}{q} \right) .
\ed
Consequently, for a fixed odd prime number $q$, the map
\bd
a \mapsto \left( \frac{a}{q} \right) : \Zbb \setminus q \Zbb \ap \bp
\ed
is multiplicative.
Further details about the Legendre symbol can be found in any elementary
text on number theory; see for example \cite[Section 6.2]{Tattersall99},
or \cite[Section 5.2.3]{BCC09}.

The LPS construction gives two distinct kinds of graphs, depending
on whether $\Leg{p}{q} = 1$ or $-1$.
To describe the situation, let us partition $\PGq$ 
into two disjoint sets $\PSq$ and $\PScq$, that are defined next.
Partition $GL(2,\Fq)$ into two sets $GL1(2,\Fq)$ and $GL2(2,\Fq)$ as follows:
\bd
GL1(2,\Fq) = \{ A \in GL(2,\Fq) : {\rm det}(A) \in S_{1,q} \} ,
\ed
\bd
GL2(2,\Fq) = \{ A \in GL(2,\Fq) : {\rm det}(A) \in S_{-1,q} \} ,
\ed
Because of the multiplicativity of the Legendre symbol, it follows that
$GL1(2,\Fq)$ is a subgroup of $GL(2,\Fq)$.
Next, define
\bd
\PSq := GL1(2,\Fq) / \sim , 
\ed
\bd
\PScq := GL2(2,\Fq) / \sim .
\ed
Then $\PSq$ and $\PScq$ form a partition of $\PGq$, and each set contains
$(q(q^2-1))/2$ elements.

Now we come to the nature of the Cayley graph that is generated by the
LPS construction.
\bit
\item If $\Leg{p}{q} = 1$, then each $M_j$ maps $\PSq$ onto itself,
and $\PScq$ into itself.
Thus the Cayley graph consists of two disconnected components,
each with $(q(q^2-1))/2$ elements.
It is shown in \cite{Lubotzky-et-al88} that each component is a Ramanujan
graph.
It is shown below that the two graphs are actually isomorphic.
\item If $\Leg{p}{q} = -1$, then each $M_j$ maps $\PSq$ onto $\PScq$,
and vice versa.
In this case the Cayley graph of the LPS construction is a balanced 
bipartite graph, with $(q(q^2-1))/2$ elements in each component.
It is shown in \cite{Lubotzky-et-al88} that the graph is a bipartite
Ramanujan graph.
It is shown below that the bipartite graph can be converted to a
non-bipartite graph.
\eit

\subsection{Gunnells' Construction}\label{ssec:32}

Next we review the construction in \cite{Gunnells05}.
Suppose $q$ is a prime or a prime power, and as usual, let $\Fq$
denote the finite field with $q$ elements.
Let $l$ be any integer, and view $\Fbb_{q^l}$ as a linear vector space
over the base field $\Fq$.
Then the number of one-dimensional subspaces of $\Fbb_{q^l}$ equals
\bd
\nu(l,q) := \frac{q^l-1}{q-1} = \sum_{i=0}^{l-1} q^i .
\ed
Let us denote the set of one-dimensional subspaces by $\V_1$.
Correspondingly, the number of subpaces of $\Fbb_{q^l}$ of \textit{codimension}
one also equals $\nu(l,q)$.
Let us denote this set by $\V_{l-1}$ because (obviously) every subspace
of codimension one has dimension $l-1$.
The Gunnells construction is a bipartite graph with $\V_1$ and $\V_{l-1}$
as the two sets of vertices.
In this construction, there is an edge between $\S_a \in \V_1$ and
$\T_b \in \V_{l-1}$ if and only if $\S_a$ is a subspace of $\T_b$.
The Gunnells construction is a balanced bipartite graph with
\bd
n = |\V_1| = |\V_{l-1}| = \nu(l,q) = \sum_{i=0}^{l-1} q^i ,
\ed
vertices and is biregular with degree
\bd
d = \nu(l-1,q) = \sum_{i=0}^{l-2} q^i .
\ed
For ease of reference, we state properties of the Gunnells construction
as a theorem.

\begin{theorem}\label{thm:Gunn}
(See \cite[Theorem 3.2]{Gunnells05}.)
The Gunnells graph is a balanced bipartite Ramanujan graph with
$\s_1 = \nu(l-1,q)$.
Moreover all other singular values of $B(l,q)$ have magnitude $\sqrt{q^{l-2}}$.
\end{theorem}

It is shown below that the bipartite graph can be converted to a
non-bipartite graph.

\subsection{Winnie Li's Construction}\label{ssec:33}

Suppose $G$ is an Abelian group and let $n$ denote $|G|$.
Then a \textbf{character} $\chi$ on $G$ is a homomorphism $\chi: G \ap S^1$,
where $S^1$ is the set of complex numbers of magnitude one.
Thus $\chi(ab) = \chi(a) \chi(b)$ for all $a,b \in G$.
There are precisely $n$ characters on $G$, and it can be shown that
$[\chi(a)]^n = 1$ for each character $\chi$ and each $a \in G$.
Therefore each character maps an element of $G$ into an $n$-th root of one.
The character defined by $\chi_0(a) = 1$ for all $a \in G$ is called the
\textbf{trivial character}.
Let us number the remaining characters on $G$ as $\chi_i, i \in [n-1]$ in
some manner.

Suppose that $S$ is a symmetric subset of $G$, and consider the
associated Cayley graph.
Thus the vertices of the graph are the elements of $G$, and the edges
are $(x,xa), x \in G , a \in S$.
Clearly the Cayley graph has $n$ vertices and is $d$-regular where
$d = |S|$.
Now a key result {\cite[Proposition 1.(1)]{Li92}}
states that the eigenvalues of the adjacency matrix are
\bd
\l_i = \sum_{s \in S} \chi_i(s), i = 0, 1 , \ldots , n-1 .
\ed
If $i = 0$, then $\l_0 = d$, which we know (due to regularity).
Therefore, if the set $S$ can somehow be chosen in a manner that
\bd
\left| \sum_{s \in S} \chi_i(s) \right| \leq 2 \sqrt{d-1} , i \in [n-1] ,
\ed
then the Cayley graph would have the Ramanujan property.
Several Ramanujan graphs can be constructed using the above approach.
In particular, the following construction is described in \cite[Section 2]{Li92}.

Let $\Fq$ be a finite field, so that $q$ is a prime or prime power.
Let $\Fbb_{q^2}$ be a degree $2$ extension of $\Fq$ and choose $S$ to be
the set of all primitive elements of $\Fbb_{q^2}$, that is, the set of elements
of multiplicative norm 1.  Here, the multiplicative norm of $\alpha \in \Fbb_{q^2}$ 
is defined as
\bd
N(\alpha) := \alpha \cdot \alpha^q = \alpha^{q+1}.
\ed

Note that $S$ is symmetric and contains $q+1$ elements.
A deep and classical theorem due to Deligne \cite{Deligne74} states that
for every nontrivial character $\chi_i, i \geq 1$ on the additive group
of $\Fbb_{q^2}$, we have
\bd
\left| \sum_{s \in S} \chi_i(s) \right| \leq 2 \sqrt{q} , i \in [q^2-1] .
\ed
Therefore the Cayley graph with $\Fbb_{q^2}$ as the group and $S$ as the
generator set has the Ramanujan property.

\subsection{Bibak et al.\ Construction}\label{ssec:34}

The next construction is found in \cite{Bibak-et-al16}.
Suppose $q$ is a prime $\equiv 3 \mod 4$.
Let $G$ be the additive group $\Fbb_{q^2}$, consisting of pairs $z = (x,y)$
where $x,y \in \Fq$.
Clearly $n = |G| = q^2$.
The ``norm'' of an element $z$ is defined as $( x^2 + y^2 ) \mod q$,
and the set $S$ consists of all $z$ such that the norm equals one.
It can be shown that $|S| = q+1$ and that $S$ is symmetric.
It is shown in \cite{Bibak-et-al16} that the associated Cayley graph has
the Ramanujan property.

\section{Further Analysis of Earlier Constructions}\label{sec:Further}

It is seen from Section \ref{sec:Rama-G} that if $\Leg{p}{q} = 1$,
then the LPS construction leads to a \textbf{disconnected} graph, consisting
of two components with an equal number of vertices (and edges).
It would be of interest to know whether the two connected components
are isomorphic to each other, that is, whether LPS construction leads
to two distinct Ramanujan graphs, or just one.
In general, determining whether two graphs are isomorphic is in the class NP,
(so that it is easy to determie whether a candidate solution is in fact
a solution).
The best available algorithm \cite{Babai-ICM18} provides a
``quasi-polynomial time'' algorithm to solve this problem.
Therefore the problem is not trivial.
However, in the case of the two components of an LPS construction,
it is easy to demonstrate that the two graphs are indeed isomorphic.
Next, if the Legendre symbol $\Leg{p}{q} = -1$,
the LPS construction leads to a \textit{bipartite} Ramanujan graph.
The Gunnells construction \textit{always} leads to a bipartite Ramanujan graph.
Graph theorists prefer non-bipartite (or ``real'') Ramanujan graphs to
bipartite Ramanujan graphs.
It is easy to show that every Ramanujan graph leads to a bipartite Ramanujan
graph.
Specifically, suppose $(\V,\E)$
is a Ramanujan graph with adjacency matrix $A$.
Define two sets $\V_r, \V_c$ to be copies of $\V$, and define an edge
$(v_{ri},v_{cj})$ in the bipartite graph if and only if there is an edge
$(v_i,v_j)$ in the original graph.
Then it is obvious that the biadjacency matrix of this bipartite graph
is also $A$.
Hence the bipartite graph is a bipartite Ramanujan graph if and only if
the original graph also has the Ramanujan property.
However, the converse is not necessarily true.
Indeed, it is not yet known whether the bipartite graphs constructed in
\cite{MSS-FOCS15} can be turned into non-bipartite graphs.
Therefore it is of some interest to show that the bipartite constructions
of LPS and Gunnels can indeed be turned into non-bipartite graphs.

Suppose that there is a balanced bipartite graph with 
vertex sets $\V_r$ and $\V_c$ and edge set $\E \seq \V_r \times \V_c$.
In order to convert this bipartite graph into a nonbipartite graph,
it is necessary and sufficient to find a one-to-one and onto map
$\pi: \V_c \ap \V_r$ (which is basically a permutation),
such that whenever there is an edge $(v_i,v_j)$ in
the bipartite graph, there is also an edge $(\pi^{-1}(v_j),\pi(v_i))$.
In this way, the ``right'' vertex set can be identified with its image
under $\pi$ and the result would be an undirected nonbipartite graph.
Moreover, it is easy to show that, if $B$ is the biadjacency matrix
of the original graph, and $A$ is the adjacency matrix of the nonbipartite
graph, then $A = B \Pi$ where $\Pi$ is the matrix representation of $\pi$.
Thus the eigenvalues of $A$ are the singular values of $B$, which implies that
if the bipartite graph has the Ramanujan property, so does the nonbipartite
graph.

With this background, in the present section we first remark that,
when $\Leg{p}{q} = 1$, the two connected components of the LPS construction
are isomorphic.
Indeed, choose any $A \in \PGq$ such that ${\rm det}(A) \in S_{-1,q}$,
and define the map $\pi : \PSq \ap \PScq$
via $X \mapsto AX$.
Then the edge incidence is preserved.

Next we show that, when $\Leg{p}{q} = -1$, the resulting LPS
bipartite graph can be mapped into a (non-bipartite) Ramanujan graph.
For this purpose we establish a preliminary result.

\begin{lemma}\label{lemma:41}
There exists a matrix $A \in GL2(2,\Fq)$ such that $A^2 \sim I_{2 \times 2}$.
\end{lemma}

\begin{proof}
Choose elements $\al,\beta,\g \in \Fq$ such that $-(\al^2 + \beta \g) \in 
S_{-1,q}$.
This is easy: Choose an arbitrary $\d \in S_{-1,q}$, arbitrary $\al \in \Fq$,
$\g = -1$, and $\beta = \al^2 + \d$.
Now define 
\be\label{eq:41}
A = \left[ \ba{cc} \al & \beta \\ \g & -\al \ea \right] .
\ee
Then
\bd
A^2 = \left[ \ba{cc} \al & \beta \\ \g & -\al \ea \right]
\left[ \ba{cc} \al & \beta \\ \g & -\al \ea \right] =
\left[ \ba{cc} \al^2 + \beta \g & 0 \\ 0 & \al^2 + \beta \g \ea \right] .
\ed
Hence $A^2 \sim I$.
Clearly ${\rm det}(A) = -(\al^2 + \beta \g) \in S_{-1,q}$.
\end{proof}

\begin{theorem}\label{thm:42}
Suppose $\Leg{p}{q} = - 1$, and consider the bipartite Ramanujan graph
$\C(\PGq,S)$.
Then there exists a map $\pi : \PSq \ap \PScq$ that is one-to-one and onto
such that, whenever there is an edge $(X,Z)$ with $X \in \PSq$
and $Z \in \PScq$, there is also an edge $(\pi^{-1}(Z),\pi(X))$.
\end{theorem}

\begin{proof}
For a matrix $\Xb \in GL(2,\Fq)$, let $[\Xb] \in \PGq$ denote its
equivalence class under $\sim$.
Construct the matrix as in Lemma \ref{lemma:41}.
Suppose there exists an edge $(X,Z)$ with $X \in \PSq$ and $Z \in \PScq$.
Thus there exist representatives $\Xb \in GL1(2,\Fq), \Zb \in GL2(2,\Fq)$
and an index $i \in [p+1]$ such that $\Zb \sim \Xb M_i$.
Now observe that if $(a_0,a_1,a_2,a_3)$ solves \eqref{eq:31}, then 
$(a_0,-a_1,-a_2,-a_3)$ also solves \eqref{eq:31}.
By examining the definition of the matrices $M_j$ in \eqref{eq:32},
it is clear that for every index $i \in [p+1]$, there exists another
index $j \in [p+1]$ such that $M_i M_j \sim I$.
Now define $\Yb = A \Zb$ and note that $\Yb \sim A^{-1} \Zb$ because
$A^2 \sim I$.
Also, $\Yb \in GL1(2,\Fq)$ because $\Zb \in GL2(2,\Fq)$ and
${\rm det}(A) \in S_{-1,q}$.
By assumption $\Zb \sim \Xb M_i$.
So $A \Zb \sim A \Xb M_i$.
Now choose the index $j \in [p+1]$ such that $M_i M_j \sim I$.
Then 
\bd
\Yb M_j = A \Zb M_j \sim A \Xb M_i M_j \sim A \Xb .
\ed
Hence there is an edge from $[ \Yb ] = \A^{-1}(Z)$ to $[ A \Xb ] = \A(X)$.
\end{proof}


\begin{theorem}\label{thm:43}
The Gunnells construction can be converted into a nonbipartite Ramanujan
graph of degree $d = \nu(l-1,q)$ and $n = \nu(l,q)$ vertices.
\end{theorem}

\begin{proof}
As before, it suffices to find a one-to-one and onto map $\pi$ from $\V_{l-1}$
to $\V_1$ such that, if there is an edge $(\S,\T)$ where $S \in \V_1$
and $\T \in \V_{l-1}$, then there is also an edge $(\pi(\T),\pi^{-1}(\S))$.
Accordingly, if $\T \in \V_{l-1}$, so that $\T$ is a subspace of codimension 
one, define $\pi(\T)$ to be an ``annihilator'' $\T^\perp$
of $\T$, consisting of all
vectors $v \in \Fq^l$ such that $v^\top u = 0$ for every $u \in \T$.
Then $\T^\perp$ is a one-dimensional subspace of $\Fq^l$ and thus belongs
to $\V_1$.
Moreover, for $\S \in \V_1$, we have that $\pi^{-1}(\S) = \S^\perp$.
It is obvious that this map is one-to-one and onto.
Now suppose there is an edge $(\S,\T)$ where $S \in \V_1$ and $\T \in \V_{l-1}$.
This is the case if and only if $\S \seq \T$.
But this implies that $\T^\perp \seq \S^\perp$.
Hence there is an edge from $\pi(\T)$ to $\pi^{-1}(\S)$.
\end{proof}

\section{Computational Aspects of Various Constructions}\label{sec:Asp}

In this section we discuss some of the implementation details of constructing
Ramanujan graphs using the various methods discussed in Section 
\ref{sec:Rama-G}.
The authors have written \texttt{Matlab} codes for all of these implementations
except the Winnie Li construction,
which can be made available upon request.
One of the objectives of this discussion is to compare and contrast the
amount of work involved in actually constructing the various classes of graphs.

\textbf{Our New Construction:}
We begin by noting
that the biadjacency matrix of the Ramanujan bigraphs presented in
Section \ref{sec:Rama-B} is quite explicit, and does not require any work:
One simply specifies a prime number $q$ and an integer $l \geq 2$, and
the biadjacency matrix is obtained at once from \eqref{eq:51}.
We now discuss the remaining constructions, in order of increasing
complexity of implementation.

For the Bibak construction, we simply enumerate all vectors
in $\Zbb_q^2$, compute all of their norms, and identify the elements of
norm one.
Again, for $q \leq 103$, this works quite well.

\textbf{Gunnells Construction:}
For the Gunnells construction, the vertex set is the set of one-dimensional
subspaces of $\Fbb_{q^l}$.
For this purpose we identify $\Fbb_{q^l}$ with $\Fq^l$, the set of 
$l$-dimensional vectors over $\Fq$.
To enumerate these, observe that there are $q^l-1$ nonzero vectors in $\Fq^l$.
For any nonzero vector, there are $q-1$ nonzero multiples of it, but all
these multiples generate the same subspace.
Therefore the number of one-dimensional subspaces is
\bd
\frac{q^l-1}{q-1} = \sum_{i=0}^{l-1} q^i .
\ed
To enumerate these subspaces without duplications, we proceed as follows:
In step 1, fix the first element of a vector $x \in \Fq^l$ to $1$,
and let the elements $x_2$ through $x_l$ be arbitrary.
This generates $q^{l-1}$ nonzero vectors that generate distinct subspaces.
In step 2, fix $x_1 = 0$, $x_2 = 1$ and $x_3$ through $x_l$ be arbitrary.
This generates $q^{l-2}$ nonzero vectors that generate distinct subspaces.
And so on.
To construct the edge set, suppose $x,y$ are two nonzero generating
vectors defined as above (which could be equal).
Then the one-dimensional subspace generated by $x$ is \textit{contained in}
the one-dimensional subspace annihilated by $y$ if and only if
$y^\top x \equiv 0 \mod q$.

\textbf{LPS Contruction:}
In order to implement this construction, it is desirable to have systematic
enumerations of the projective groups $\PGq$, $\PSq$, and $\PScq$.
The approach used by us is given next.
Define
\bd
\M_1 = \left\{ \left[ \ba{cc} 0 & 1 \\ g & h \ea \right] : g \in \Fqs,
h \in \Fq \right\} ,
\ed
\bd
\M_2 = \left\{ \left[ \ba{cc} 1 & f \\ g & h \ea \right] : f , g \in \Fqs,
h - fg \in \Fqs \right\} .
\ed
Then every matrix in
$GL(2,\Fq)$
is equivalent under $\sim$ to exactly one element
of $\PGq$.
Specifically, let $A \in GL(2,\Fq)$ be arbitrary.
If $a_{11} = 0$, then $A \sim (a_{12})^{-1} A \in \M_1$, whereas if
$a_{11} \neq 0$, then $A \sim (a_{11})^{-1} A \in \M_2$.
The rest of the details are easy and left to the reader.
This provides an enumeration of $\PGq$.
To provide an enumeration of $\PSq$, we modify the sets as follows:
\bd
\M_{1,1} = \left\{ \left[ \ba{cc} 0 & 1 \\ g & h \ea \right] : g \in S_{1,q},
h \in \Fq \right\} ,
\ed
\bd
\M_{2,1} = \left\{ \left[ \ba{cc} 1 & f \\ g & h \ea \right] : f , g \in \Fqs,
h - fg \in S_{1,q} \right\} ,
\ed
\bd
\M_{1,-1} = \left\{ \left[ \ba{cc} 0 & 1 \\ g & h \ea \right] : g \in S_{-1,q},
h \in \Fq \right\} ,
\ed
\bd
\M_{2,-1} = \left\{ \left[ \ba{cc} 1 & f \\ g & h \ea \right] : f , g \in \Fqs,
h - fg \in S_{-1,q} \right\} .
\ed
Then the set of matrices $\M_{1,1} \cup \M_{2,1}$ provides an enumeration of
$\PSq$, while the set of matrices $\M_{1,-1} \cup \M_{2,-1}$ provides an
enumeration of $\PScq$.
Once the vertex sets are enumerated, each representative matrix
of an element in $\PGq$ is multiplied by each generator matrix $M_1$
through $M_{p+1}$, \textit{and then converted to one of the above
representations}, depending on whether the $(1,1)$-element is zero or nonzero.

In the original LPS construction, it is assumed that $p < q$.
However, the LPS construction can still be used with $p > q$,\footnote{We
are grateful to Prof.\ Alex Lubotzky for clarifying this point.}
provided that the $p+1$ generating matrices $M_1$ through $M_{p+1}$
are distinct elements of $\Fq^{2 \times 2}$.
In our implementation, we handle the case $p > q$ by verifying whether this 
is indeed the case.

\textbf{Winnie Li Construction:}
As mentioned earlier, the fact that the Winnie Li construction leads
to a Ramanujan graph is based on a deep result of \cite{Deligne74}.
However, it is still
a challenge to construct this graph explicitly.
The main difficulty is that as yet there is no polynomial-time algorithm
to find all elements of the generator set $S$ of a finite field $\Fbb_q$.
Suppose that $q = p^r$ where $p$ is a prime number and $r$ is an integer.
Then the best available algorithm \cite{Shparlinski96} returns
\textit{one} primitive element of the field $\Fbb_q$ in time
$O(p^{(r/4)+\e})$ for arbitrarily small $\e$.
Recall that, for an algorithm for integer computations to be considered
as ``polynomial-time,'' its running time has to be $O(\lceil \log_2 q \rceil)$.
Thus, for the moment there are no efficient algorithms for implementing the
Winnie Li construction.
However, for relatively small values of $q$ (say $101$ or less), 
it is still feasible to execute ``nonpolynomial'' algorithms, basically
just enumeration of all possibilities.
The approach adopted by us to implement this construction is described next.

This construction requires more elaborate computation.
The main source of difficulty is that as of now there is no
polynomial-time algorithm for finding a primitive element in a finite field.
The best known algorithm to date is due to \cite{Shparlinski96},
which finds a primitive element of $\Fbb_{p^l}$ in time $O(p^{(l/4)+\e})$.
Thus, for really large primes $q$, the Li construction would be difficult
to implement.
At present, ``realistic'' matrix completion problems such as the Netflix
problem could potentially have millions of columns.
However, many if not most practical examples are of size $10^4 \times 10^4$.
Since the graphs in the Li construction have $q^2$ vertices, it is reasonable
to focus on the much narrower problem of finding a primitive element in
the field $\Fq$ when $q$ is a prime $\leq 101$.
For this purpose we follow a simple enumeration procedure.

\ben
\item
We identify an irreducible polynomial $\phi(x)$ of degree $2$ in $\Fq[x]$.
For instance, if $q \equiv 5 \mod 8$, then $x^2 + 2$ is irreducible. 
In particular, $x^2 + 2$ is an irreducible polynomial in $\Fbb_{13}[x]$.
Other irreducible polynomials are known for other classes
of prime numbers.
In the worst case, an irreducible polynomial can be found by enumerating all
polynomials of degree $2$ in $\Fq[x]$, and then deleting all products
of the form $(x-a)(x-b)$ where $a,b \in \Fq$.
This would leave $(q(q-1))/2$ irreducible polynomials, but we can use
any one of them.
Note that the worst-case complexity of this approach is $O(q^2)$.
\item Once this is done, $\Fbb_{q^2}$ is isomorphic to the quotient field
$\Fq[x]/(\phi(x))$.
Let us denote $\overline{x} = x + (\phi(x))$.
Then the elements of $\Fbb_{q^2}$ can be expressed as
\bd
a \overline{x} + b, \, a, \, b \in \Fbb_{q}.
\ed
\item To determine the set of units $S$ in the field $\Fbb_{q^2}$,
we have to find all elements $\alpha$
in $\Fbb_{q^2}^*$ such that $\alpha^{q+1} = 1$.
Now, if $\beta$ is a primitive element of the $\Fbb_{q^2}^*$, then
\bd
\Fbb_{q^2}^* =  \{\beta^{m}, \, 1 \leq m \leq q^2-1\}.
\ed
Among the elements of $\Fbb_{q^2}^*$,
the units are precisely the roots of $\alpha^{q+1} = 1$.
With the above enumeration, we get
\bd
S = \{\beta^{k(q-1)},\,1 \leq k \leq q+1\}.
\ed
\item We now enumerate all elements in the multiplicative group $\Fbb_{q^2}$
and test each element to see if it is a primitive element.
To ensure that $g \in \Fbb_{q^2}^*$ is a primitive element,
it is sufficient to check that
$g^{\frac{q^2 - 1}{l}} \neq 1$ for any prime divisor $l$ of $q^2 - 1$.
Because $q^2 - 1 = (q-1)(q+1)$, the prime divisors of $q^2-1$ are $2$
and the other prime divisors of $q-1$ and $q+1$.
For example, with $q = 13$, the prime divisors of $q^2-1$ are $l = 2, 3, 7$,
with the corresponding values $k = (q^2-1)/l = 84, 56, 24$.
Therefore if $g^k \neq 1$ for these values of $k$, then $g$ is a primitive
element.
There are also efficient ways to raise a field element to
large powers, for example, by expanding the power to the base $2$.
One can check that $g = \overline{x} + 1$ is a primitive element of the field
$\Fbb_{13^2}$.
\item Once a primitive element $g \in \Fbb_{q^2}$ is determined,
the generator set $S$ equals $\{ g^{(q-1)k} , k \in [q+1] \}$.
Thus we can construct the Cayley graph $\C(\Fbb_{q^2},S)$ as follows:
Each vertex $a\overline{x} + b$ is connected to the vertices
\bd
\{a\overline{x} + b + g^{(q-1)k}, k \in [q+1] \}.
\ed
In this way we can construct a $(q+1)$-regular Ramanujan graph with $q^2$
vertices.
\een
Note that, in the worst case, the Winnie Li construction \textit{cannot be
implemented by any poly-time algoritheorem}.
The above are just some shortcuts that may or may not work for specific
values of $q$.

\section{Construction of Ramanujan Graphs with ``Prohibited''
Edges}\label{sec:Pert}

\subsection{Problem Statement and Motivation}\label{ssec:71}

Suppose we are given vertex sets $\V_r, \V_c$, where $|\V_r|$ need not
equal $|\V_c|$.
Denote $|\V_r| = n_r, |\V_c| = n_c$.
We are also given degrees
$d_r, d_c$, as well as a set $\M \seq \V_r
\times \V_c$, known as \textbf{the set of prohibited edges}.
The objective is to construct a bipartite Ramanujan graph of degrees
$\drc$, but with the added restriction that the edge set $\E$ must be
disjoint from $\M$.

The motivation for this problem arises from an engineering
application known as ```matrix completion with missing measurements.''
The conventional matrix completion problem can be stated as follows:
Suppose $X \in \Rno$ is a matrix that is unknown other than a known
upper bound $r$ on its rank.
The learner is able to choose a ``sample set'' $\OM \seq [n_r] \times [n_c]$,
and can measure $X_{ij}$ for all $(i,j) \in \OM$.
From these measurements, and the information that $\rk(X) \leq r$,
the learner aspires to determine the unknown matrix $X$ exactly.
The most widely used approach, introduced in \cite{Candes-Recht08,RFP10},
is the following:
Let $\nmN{X}$ denote the \textbf{nuclear norm} of a matrix, namely the
sum of its singular values.
Given the measurement set $X_{ij}, (i,j) \in \OM$, define
\be\label{eq:71}
\Xh := \argmin_{Z \in \Rno} \nmN{Z} \st Z_{ij} = X_{ij} \fa (i,j) \in \OM .
\ee
Thus one finds the matrix $\Xh$ that has minimum nuclear norm 
while being consistent with the measurements.
The problem \eqref{eq:71} is a convex optimization problem subject to
linear constraints; so it can be solved efficiently even for relatively
large-sized matrices.
However, the key question is: Is the solution $\Xh$ of \eqref{eq:71} equal
to the unknown matrix $X$?
To put it another way, can an unknown matrix of low rank be recovered
by sampling a few entries and then carrying out nuclear norm minimization?

It is clear that the choice of the sample set $\OM$ plays a crucial role
in any analysis of the matrix completion problem.
In almost all papers, starting with \cite{Candes-Recht08}, it is assumed
that $\OM$ is determined by sampling the elements of $X$ uniformly and at
random.
Then it is shown that, if the number of samples $|\OM|$ is sufficiently
large, then with high probability $\Xh$ does indeed equal $X$.
A significant departure is made in a recent paper by a subset of the
present authors \cite{Shantanu-TSP20}, where the sample set $\OM$ is
chosen in a \textit{deterministic} fashion,
as the edge set of a Ramanujan bigraph.
In \cite{Shantanu-TSP20} a sufficient condition is derived for
the solution of \eqref{eq:71} to equal $X$.

While \cite{Shantanu-TSP20} gives the first deterministic
solution to the matrix completion problem, it does not address
the issue of ``missing measurements.''
In any real-life application of matrix completion, it is a fact that
\textit{a few}
elements of the matrix $X$ are ``missing'' and thus cannot be measured.
This is the set $\M \seq [n_r] \times [n_c]$.
Therefore, in applying nuclear norm minimization, it is essential to
ensure that the ``sample set'' $\OM$ is disjoint from the missing
measurements, that is, the set $\M$ of prohibited edges.
If the elements of $X$ are sampled at random, as suggested
in \cite{Candes-Recht08,RFP10},
then one can simply omit the elements of $\M$ while sampling.
However, if a Ramanujan graph is constructed using one of the methods
outlined in earlier sections, it is not possible to guarantee
\textit{a priori} that the resulting edge set $\OM$ and the
prohibited edge set $\M$ will be disjoint.

This now brings us to a formal statement of the missing measurement problem in
Ramanujan bigraphs.
Suppose that we have constructed a $\drc$-biregular Ramanujan bigraph
with $n_r, n_c$ verticies.
This graph can be described in one of two equivalent ways:
First, its biadjacency matrix $E \in \bi^{n_r \times n_c}$,
and second, its edge set $\OM$ which consists of all $(i,j)$ such that
$E_{ij} = 1$.
It is tacitly understood that the construction might impose some constraints
on the four numbers $d_r, d_c, n_r , n_c$.
Next, suppose that the set of prohibited edges
$\M \seq [n_r] \times [n_c]$ is also specified.
In applications, the set $\M$ has quite small cardinality.
If $\OM$ and $\M$ are disjoint sets, then there is nothing to be done.
On the other hand, if $\OM$ and $\M$ have some overlap,
the question is: Is it possible to ``perturb'' the biadjacency matrix $E$
in such a manner that the corresponding graph has no edges in the set $\M$,
and is still a $\drc$-biregular Ramanujan bigraph?
Note that we could have phrased the problem differently:
Suppose integers $d_r, d_c, n_r, n_c$ and a set $\M \seq [n_r] \times [n_c]$
are specified.
Is it possible to construct a $\drc$-biregular Ramanujan graph such that
no element of $\M$ is an edge?
Our method of framing the question suggests our proposed approach:
We first construct a Ramanujan bigraph without worrying about the set $\M$,
and \textit{afterwards} worry about ``perturbing'' the Ramanujan bigraph
so as to eliminate the edges in $\M$ (if any) and replace them by others,
while still retaining biregularity and the Ramanujan property.
Since the underlying assumption (motivated by real-life applications)
is that the set $\M$ is quite small, this approach makes sense.

\subsection{Perturbations of Ramanujan Graphs}\label{sec:72}

Let
\bd
\M = \{(i_1,j_1),\,(i_2,j_2),\dots \,(i_s,j_s)\} \seq [n_r] \times [n_c]
\ed
denote the set of prohibited edges.
Now suppose we have constructed a $\drc$-biregular
Ramanujan graph with biadjacency matrix $E$ and edge set $\OM$.
If $\OM \cap \M = \es$, then there is nothing for us to do.
If on the other hand the intersection is nonempty, then we modify $E$
to another matrix $E_p = E + \D_+ - \D_-$ where
$\D_- , \D_+ \in \bi^{n_r \times n_c}$ such that $E_p \in \bi^{n_r \times n_c}$,
and also corresponds to a $\drc$-biregular Ramanujan graph.
There are two distinct issues here: Ensuring that $E_p \in
\bi^{n_r \times n_c}$ and is $\drc$-biregular,
and ensuring that $E_p$ represents a Ramanujan graph.
The two questions are treated sequentially.
We begin with perturbing the graph to ensure $\drc$-biregularity while
avoiding prohibited edges.

To recapitulate, given a matrix $E$ and corresponding edge set $\OM$,
the problem is to choose $\D_- , \D_+ \in \bi^{n_r \times n_c}$ such that
$E_P := E + \D_+ - \D_-$
is also $\drc$-biregular while avoiding prohibited edges.
This requires, among other things, that
\be\label{eq:71a}
\supp(\D_-) \seq \OM , \mbox{ and } \OM \cap \M \seq \supp(\D_-) ,
\ee
and 
\be\label{eq:71b}
\supp(\D_+) \cap \OM = \es, \mbox{ and } \supp(\D_+) \cap \M = \es .
\ee
Here, \eqref{eq:71a} ensures that $\D_-$ is a submatrix of $E$, and that by 
subtracting the edges in $\supp(\D_-)$ from $\OM$, we eliminate all
prohibited edges.
Note that we \textit{do not insist} that $\OM \cap \M = \supp(\D_-)$.
In other words, in the process of eliminating some prohibited edges,
we may perhaps also remove other edges that are not prohibited.
Clearly this does not matter.
Along similar lines, \eqref{eq:71b} ensures that only new edges are added,
and that these edges are not prohibited.

Lemma \ref{lemma:72} presented next gives a very simple sufficient condition and
a recipe for choosing $\D_-,\D_+$ provided that $\M$ satisfies certain conditions. 
It can definitely be improved, but in the case of matrix completion with
missing measurements, it is good enough.

\begin{lemma}\label{lemma:72}
Suppose that, for each index $i \in [n_r]$, the set
	$\{ j \in [n_c]: (i,j) \in \M \}$ has cardinality no more than $p$,
	and similarly, for each index $j \in [n_c]$, the set
	$\{ i \in [n_r]: (i,j) \in \M \}$ has cardinality no more than $p$.
Let $\th_c$ denote the maximum inner product between any two distinct
columns of $E$.
Suppose it is the case that
\be\label{eq:79}
2p \leq n_c - d_r ,
\ee
\be\label{eq:710}
2 p - 1 \leq d_c - \th_c.
\ee
Then we can choose $\D_-, \D_+ \in \bi^{n_r \times n_c}$ such that
\eqref{eq:71a} and \eqref{eq:71b} hold.
\end{lemma}

\begin{proof}
We introduce a few symbols to streamline the presentation.
In analogy with the notation introduced earlier, define
\bd
\supp(i) := \{ j \in [n_c] : E_{ij} = 1 \},
\ed
\bd
\supp(j) := \{ i \in [n_r] : E_{ij} = 1 \},
\ed

\be\label{eq:711}
\N_r(i) : = \{ j \in [n_c] : (i,j) \in \OM \cap \M \}, \fa i \in [n_r] .
\ee
\be\label{eq:712}
\N_c(j) : = \{ i \in [n_r] : (i,j) \in \OM \cap \M \} , \fa j \in [n_c] .
\ee

Order the elements of $[n_r],[n_c]$ in some fashion.
Ascending order is as good as any, and it will play no role in the proof
other than simplifying notation.
So we use that.
%
Apply the procedure below for each $i \in [n_r]$, in sequential order.
If the set $\N_r(i)$ is empty, move to the next row.
Suppose $\N_r(i)$ is nonempty, and define $p_i = | \N_r(i) |$.
Since $p_i \leq p$ for each $i$, we simplify notation by using $p$
instead of $p_i$.

Recall the definition of the set $\N_r(i)$.
Enumerate the set as $\{ (i,j_1) , \cdots , (i,j_p) \}$.
and observe that $E_{i,j_l} = 1$ for all $l$.
Note that row $i$ contains precisely $d_r$ columns with a $1$
(including of course $E_{ij_1}$), and therefore $n_c - d_r$ columns with zeros.
Among these, a maximum of $p$ can be such that $(i,j) \in \M$.
Therefore, if $n_c - d_r \geq 2p$, then we can find indices $\jb_l,
1 \leq l \leq p$ such that
\bd
E_{i,\jb_l} = 0, \mbox{ and } (i,\jb_l) \not\in \M , 1 \leq l \leq p .
\ed

Next,  fix a specific index $\jb_l$, and compare the columns
$E_{j_l}$, and $E_{\jb_l}$.
It is now shown that it is possible to find a subset
	$\I_{\jb_l} = \{ \ib_1 = i , \ib_2 , \cdots, i_p \} \seq [n_r]$
	such that
\be\label{eq:713}
	E_{\ib_k \jb_l} = 1 , k = 1 , \cdots, p ,
\ee
\be\label{eq:714}
	E_{\ib_k j_l} = 0 , k = 1 , \cdots, p .
\ee
\be\label{eq:715}
| \I_{\jb_l} | = p ,
\ee
\be\label{eq:716}
(\ib_k , j_l) \not\in \M , 1 \leq k \leq p .
\ee
To establish this, we observe that, by assumption, we have that
$\IP{E_{j_l}}{E_{\jb_l}} \leq \th_c$.
So, as a first cut, define $\I_1(\jb_l) = \supp(\jb_l)$.
This set consists of $d_c$ elements (and \textit{does not include} $i$).
Now examine the values $E_{i_k\jb_l}$.
By the inner product constraint, out of the $d_c$ nonzero elements of
$\I_1(\jb_l)$,
no more
than $\th_c$ can be nonzero (in column $j_1$).
Hence there are at least $d_c - \th_c$ elements in column $j_l$
such that $E_{i_kj_l} = 0$ and $E_{i_k \jb_l} = 1$.
Out of these, at most $p-1$ elements can have the property that
$(i_k,j_l) \in \M$.
Note that $E_{ij_1} = 0$, but $i$ is not
on this list.
Therefore, if $d_c - \th_c - (p - 1) \geq p$, that is, $d_c - \th_c \geq 2p - 1$,
then we can choose any $p$ elements from whatever remains of the set
$\I_1(\jb_l)$ and call it the set $\I(\jb_l)$.
Now the above-claimed properties follow by inspection.

To complete the proof, we remove each prohibited edge $E_{ij_l}$ as follows: We choose any one element from $\I(\jb_l)$, call it $\ib_l$ and simply ``flip'' the zeros and ones.
Thus we set
\be\label {eq:715a}
(E_p)_{ij_l} = 0 , (E_p)_{\ib_l j_l} = 1,
\ee
\be\label{eq:715b}
(E_p)_{\ib_l j_l} = 0 ,\,(E_p)_{i \jb_l} = 0.\ee
Since the set $\I(\jb_l)$ has $p$ elements for each $l$, we can ensure that the the vertices $\ib_l$ chosen for each $l$ are distinct.  
If we now write $E_p$ as $E - \D_- + \D_+$, it is evident that
the three claims in the lemma follow upon observation.
The process can be repeated for each row $i$.
\end{proof}

Now we present our result on perturbing a Ramanujan bigraph while
retaining the Ramanujan property.

\begin{theorem}\label{thm:71}
Suppose $E$ is the biadjacency matrix of a $\drc$-biregular Ramanujan bigraph,
and define
\be\label{eq:72}
\mu(d_r,d_c) := \sqrt{d_r-1} + \sqrt{d_c-1} .
\ee
Let $\s_2$ denote the second largest singular value of $E$.
Next, suppose the set $\M$ of prohibited edges does not contain more than
$p$ values of $j$ for each fixed $i$, and more than $p$ values of $i$
for each fixed $j$.
Finally, suppose that
\be\label{eq:75}
2p \leq \mu(d_r,d_c) - \s_2 .
\ee
Then the construction in Lemma \ref{lemma:72} results in a $\drc$-biregular
Ramanujan graph whose edge set is disjoint from the set of prohibited
edges $\M$.
\end{theorem}

The proof of Theorem \ref{thm:71} makes of the following alternate
characterization of a Ramanujan bigraph.
Note that $\nmS{\cdot}$ denotes the spectral norm of a matrix, that is,
its largest singular value.

\begin{lemma}\label{lemma:71}
Suppose $E \in \bi^{n_r \times n_c}$.
Then 
\ben
\item The corresponding graph is $\drc$-biregular if and only if
\be\label{eq:76}
E \oneb_{n_c} = d_r \oneb_{n_r} , \oneb_{n_r}^\top E = d_c \oneb_{n_c}^\top  ,
\ee
where $\oneb_k$ denotes the column vector of $k$ ones.
\item Define the constant
\be\label{eq:77}
\al = \frac{d_r}{n_c} = \frac{d_c}{n_r} = \sqrt{ \frac{ d_r d_c}{n_r n_c} } .
\ee
Suppose that \eqref{eq:76} holds, so that the graph is $\drc$-biregular.
Then the graph has the Ramanujan property if and only if
\be\label{eq:78}
\nmS{ E - \al \oneb_{n_r \times n_c} } \leq 
\sqrt{d_r-1} + \sqrt{d_c-1} =: \mu(d_r,d_c) .
\ee
where
$\oneb_{k \times l}$ denotes the $k \times l$ matrix of all ones.
\een
\end{lemma}

\begin{proof}
The first statement is obvious.
To prove the second, observe that due to biregularity, the largest singular
value of $E$ is $\s_1(E) = \sqrt{d_r d_c}$, with associated (normalized)
singular vectors $(1/\sqrt{n_r}) \oneb_{n_r}$ and
$(1/\sqrt{n_c}) \oneb_{n_c}$.
Thus, the singular value decomposition (SVD) of $E$ looks like
\bd
E = \sqrt{ \frac{ d_r d_c}{n_r n_c} } \oneb_{n_r} \oneb_{n_c}^\top + M
= \al \oneb_{n_r \times n_c} + M ,
\ed
where $\nmS{M} = \s_2(E)$, the second largest singular value of $E$.
Now \eqref{eq:78} follows from the definition of a Ramanujan bigraph.
\end{proof}

\begin{proof}
(Of Theorem \ref{thm:71}.)
The perturbed matrix $E_p$ can be written as $E - \D_- + \D_+$.
Note that since $E$ and $E_p$ are both $\drc$-biregular,
the quantity $\al$ defined in \eqref{eq:77} is the same for both
$E$ and $E_p = E - \D_- + \D_+$.
Moreover, by construction, both $\D_-$ and $\D_+$ have no more than $p$
entries of one in each row and each column.
Therefore, by Lemma \ref{lemma:71}, it follows that $E - \D_- + \D_+$
represents a $\drc$-biregular Ramanujan bigraph if
\bd
\nmS{ (E - \D_- + \D_+) - \al \oneb_{n_r \times n_c} } \leq \mu(d_r,d_c) .
\ed
We already know from Lemma \ref{lemma:71} that
\bd
\nmS{ E - \al \oneb_{n_r \times n_c} } = \s_2(E) .
\ed
Therefore the result follows if it can be established that
\bd
	\nmS{\D_-} , \nmS{\D_+}  \leq p .
\ed
Note that each of
$\D_+$ and $\D_-$ can be viewed as biadjacency matrices of bipartite graphs
of maximum degree $p$. 
Let $\sigma$ be a singular value of $\D_+$ and let 
\bd
{v} = [v_1 \, v_2 \dots v_{n_c}]^\top
\ed
be an eigenvector of $\D_+^\top\D_+$ corresponding to the eigenvalue $\sigma^2$.
That is,
\bd
(\D_+^\top\D_+) {v} = \sigma^2 {v}.
\ed
We assume without loss of generality that 
$$|v_1| = \max_{1 \leq i \leq n_c} |v_i|.$$
We now have,
\bd
\sigma^2 v_1 = R^1(\D_+^\top\D_+)v,
\ed
where $R^i(A)$ denotes the $i$-th row of a matrix $A$.

We also observe that 
\bd
\D_+^\top\D_+ = [a_{ij}]_{1 \leq i,j \leq n_c},
\ed
where $a_{ij}$ equals the dot product $\IP{C_i}{C_j}$
of $C_i$ and $C_j$, the $i$-th and $j$-th columns of $\D_+$ respectively.  
This gives us
\bd
\sigma^2 v_1 = \sum_{j=1}^{n_c} \IP{C_1}{C_j}v_j.
\ed
That is,
\bd
\sigma^2 |v_1| \leq \sum_{j=1}^{n_c} \IP{C_1}{C_j}|v_j|
\leq |v_1| \sum_{j=1}^{n_c}\IP{C_1}{C_j} .
\ed
We observe that each element of the vector $\sum_{j=1}^{n_c}C_j$ is at most $p$.
Also, the vector $C_1$ has at most $p$ number of $1$'s. 
In other words, the dot product 
\bd
\IP{C_1}{\sum_{j=1}^{n_c}C_j} \leq p^2 .
\ed

Thus, $\sigma^2 |v_1| \leq p^2 |v_1|$ and therefore, $\sigma \leq p$ for any singular value $\sigma$ of $\Delta_+$.  We have used here, a technique from the proof of \cite[Theorem 1]{Ram-Murty03}.  A similar analysis works for $\Delta_-$.  Thus, it follows that $ \nmS{\D_+} \leq p , \nmS{\D_-} \leq p$
and therefore,
\bd
\nmS{\D_+ - \D_-} \leq 2p .
\ed
This completes the proof.
\end{proof}

\subsection{Analysis of the Spectra of Ramanujan Graphs and Bigraphs}
\label{ssec:73}

Lemma \ref{lemma:71} shows that the greater the gap between $\s_2(E)$
and the bound $\mu(d_r,d_c)$, the more perturbation the Ramanujan graph
can withstand while still remaining a Ramanujan graph.
At first glance this may appear to be a vacuous result, because
the Alon-Boppana bound for graphs and the Feng-Li bound for bigraphs
shows that, as the graphs grow in size, the gap approaches zero.
Indeed, one of the major contributions of \cite{Lubotzky-et-al88} is to
construct \textit{infinitely many} families of Ramanujan graphs,
with size approaching infinity, but fixed degree.
However, in applications such as the matrix completion problem,
the Alon-Boppana and Feng-Li bounds are not germane.
As shown in \cite[Section 7]{Shantanu-TSP20}, as the size of the
unknown
matrix (which is the size of the graph) grows, so must the degree.

Both our new construction of the Ramanujan bigraphs
and the Gunnells construction of Ramanujan graphs achieve a substantial gap
between the degree $d$ and the second largest singular value $\s_2$.
In the Gunnells construction, we get
\bd
d = \sum_{i=0}^{l-2} q^i \approx q^{l-2} , 
 \mu(d,d) = 
2 \sqrt{d-1} = 2 \sqrt{ \sum_{i=1}^{l-2} q^i } \approx 2 \sqrt{q^{l-2}} ,
\s_2 = \sqrt{q^{l-2}} ,
\ed
because, when $q$ is large, the last term in various summations dominates
the rest.
Therefore
\bd
\mu(d,d) - \s_2 \approx \sqrt{d-1} .
\ed
Similarly, for our construction in Theorem \ref{thm:51}, with $l = q$,
we have that
\bd
d = q , \s_2 = \sqrt{q} , 
 \mu(d,d) = 2 \sqrt{d-1} , \mu(d,d) - \s_2 \approx
\sqrt{d-1} .
\ed
It is also worth mentioning that if a bipartite graph
has girth six or more, then the maximum inner product $\th_c = 1$.
Therefore \eqref{eq:79} and \eqref{eq:710} can be combined into
the single bound
\bd
2p \leq \min \{ n_c - d_r , d_c \} .
\ed
It is clear that the operative bound is the second term $d_c$.
To illustrate these bounds, suppose we construct a Ramanujan graph
of size $n = q^2 = 101^2$ and degree $d = q = 101$.
Then we can avoid \textit{any} set of prohibited edges provided there
are fewer than $50$ edges in each row and each column, and still
preserve $(q,q)$-biregularity.
Moreover, since $\mu(q,q) - \s_2 \geq 10$, we can avoid any arbitrary
set of prohibited edges, provided that there are not more than ten 
such edges in each row and each column.

One noteworthy feature of both our new construction and the Gunnells
construction is the high multiplicity of the second largest singular value.
In our construction, other than the largest singular value and zero (due to
rank deficiency), \textit{all} other singular values have the same value.
In the case of the Gunnells construction, other than the largest singular
value, all other singular values have the same magnitude, though some
could be negative.
This prompted us to analyze whether the LPS and Bibak-et-al.\
constructions have a similar feature.
It appears that, even in the case of the LPS construction,
the second largest singular value is repeated multiple times.
Table \ref{table:1} shows the results for $q = 13$ and $q = 17$, for
various values of $p$.
The upper limit $p = 157$ was chosen because \textit{every} prime
$p \equiv 1 \mod 4$ leads to a simple graph (and not a multigraph),
for both values of $q$.
We have thus far not found any pattern in the multiplicity of $\s_2$
as a function of $p$ and $q$; perhaps this can be a topic for future research.
Another noteworthy aspect is that $\s_2$ is very close to the Ramanujan
bound of $2 \sqrt{d-1}$ in all cases.
In contrast, the Bibak-et-al.\ construction did not show that $\s_2$
had multiplicity more than one, for any $p \equiv 3 \mod 4$, for $p$
up to $103$.
We did not test the Winnie Li construction due to the complexity of
actually computing the constructions.

\begin{table}
\bc
\btab{|r|r|r|r|r||r|r|r|r|r|}
\hline
\multicolumn{5}{|c||}{$q = 13, n = 1092$} &
\multicolumn{5}{|c|}{$q = 17, n = 2448$} \\
\hline
$p$ & $d$ & $2 \sqrt{d-1}$ & $\s_2$ & Mult. &
$p$ & $d$ & $2 \sqrt{d-1}$ & $\s_2$ & Mult. \\
\hline
5 & 6 & 4.4721 & 4.2497 &  36 & 5 & 6 & 4.4721 & 4.3089 &  48 \\
17 &  18 & 8.2462 & 7.8509 &  24 &  13 &  14 & 7.2111 & 7.0902 &  48 \\
29 &  30 &  10.7703 & 9.9323 &  36 &  29 &  30 &  10.7703 &  10.0000 &  51 \\
 37 &  38 &  12.1655 &  11.3081 &  42 &  37 &  38 &  12.1655 &  11.9855 &  16 \\
 41 &  42 &  12.8062 &  11.4940 &  24 &  41 &  42 &  12.8062 &  11.5321 &  48 \\
 53 &  54 &  14.5602 &  12.2462 &  42 &  53 &  54 &  14.5602 &  13.8995 &  36 \\
 61 &  62 &  15.6205 &  13.9758 &  12 &  61 &  62 &  15.6205 &  14.5826 &  32 \\
 73 &  74 &  17.0880 &  15.3693 &  42 &  73 &  74 &  17.0880 &  16.0192 &  48 \\
 89 &  90 &  18.8680 &  17.8215 &  42 &  89 &  90 &  18.8680 &  16.4721 &  51 \\
 97 &  98 &  19.6977 &  17.8078 &  42 &  97 &  98 &  19.6977 &  18.3848 &  54 \\
101 & 102 &  20.0998 &  19.0000 &  28 & 101 & 102 &  20.0998 &  18.9750 &  48 \\
109 & 110 &  20.8806 &  19.0000 &  13 & 109 & 110 &  20.8806 &  20.2648 &  48 \\
113 & 114 &  21.2603 &  20.6504 &  12 & 113 & 114 &  21.2603 &  20.5253 &  16 \\
149 & 150 &  24.4131 &  21.6746 &  24 & 149 & 150 &  24.4131 &  23.6574 &  48 \\
157 & 158 &  25.0599 &  22.4924 &  39 & 157 & 158 &  25.0599 &  23.8885 &  51 \\
\hline
\etab
\ec
\caption{The second largest singular value of the
LPS graph (bipartite or otherwise), and its multiplicity}
\label{table:1}
\end{table}

\section{Conclusions}\label{sec:Conc}
As described in Section \ref{ssec:71}, a significant motivation for the study of explicit constructions of Ramanujan graphs comes from the matrix completion problem.  This theme was explored in detail in \cite{Shantanu-TSP20} where it was shown that it is possible to
guarantee \textit{exact completion} of an unknown low rank matrix,
if the sampling set corresponds to the edge set of a Ramanujan bigraph.
While that set of results is interesting in itself, it has left open the
question of just \textit{how} Ramanujan bigraphs are to be constructed.
In the literature to date, there are relatively few \textit{explicit}
constructions of Ramanujan graphs and balanced bigraphs, and no explicit constructions of
an unbalanced Ramanujan \textit{bigraph}.
In this paper, we presented for the first time an infinite family of
unbalanced Ramanujan bigraphs with explicitly constructed biadjacency matrices.
In addition, we have also shown how to construct the adjacency matrices for
the currently available families of Ramanujan graphs.
These explicit constructions, as well as forthcoming ones based on
\cite{Ballantine-et-al15,Evra-Parz18}, are available for only a few
combinations of degree and size.
In contrast, it is known from \cite{MSS-Annals15a} and
\cite{MSS-FOCS15} that Ramanujan graphs are known to exist for all degrees
and all sizes.
The main limiting factor is that these are only existence proofs and
do not lead to explicit constructions.
A supposedly polynomial-time algorithm for constructing Ramanujan graphs
of all degrees and sizes is proposed in \cite{Cohen16}.
But it is still a conceptual algorithm and no code has been made available.
Therefore it is imperative to develop efficient implementations of the
ideas proposed in \cite{MSS-FOCS15}, and/or to develop other methods to
construct Ramanujan graphs of most degrees and sizes.
We should also note that  the work of \cite{MSS-Annals15a} and
\cite{MSS-FOCS15} shows the existence of \textit{balanced, bipartite} graphs of all degrees and sizes.  Therefore, in Section \ref{sec:Further} of this article, we have also looked at how the constructions of \cite{Lubotzky-et-al88} and Gunnells \cite{Gunnells05} can be further analysed to derive non-bipartite Ramanujan graphs.

It is worth pointing out that efficient solutions
of the matrix completion problem do not really require the existence
of Ramanujan graphs of \textit{all} sizes and degrees.
It is enough if the ``gaps'' in the permissible values for the degrees and
the sizes are very small.
If this extra freedom leads to substantial simplification in the
construction procedures, then it would be a worthwhile tradeoff.
However, research on this problem is still at a nascent stage.

Finally, in Section \ref{sec:Pert}, we address another issue in the matrix completion problem, namely the ``missing measurements" problem.  This leads to the problem of the construction of a Ramanujan bigraph (not necessarily balanced) in which a certain set of edges is prohibited.  In a typical real-life application, the size of this set is small.  In Section \ref{sec:72}, we initiate the study of how a Ramanujan bigraph can be perturbed so as to remove a set of prohibited edges while still retaining biregularity and the Ramanujan condition.  The set of prohibited edges considered is relatively small in size and has some additional constraints. 
Nonetheless, this study is at an initial stage and ``perturbing" a Ramanujan bigraph by the set $\M$ with weaker conditions than we currently require would be an interesting topic for future research.

\section*{Acknowledgements}

The authors thank Professors Cristina Ballantine and Alex Lubotzky for
helpful discussions during the preparation of this article. 

\section*{Declarations}

S.P.B.  was supported through a Senior Research Fellowship from the Ministry of
Human Resource Development, Government of India.  
K.S. was supported by a MATRICS grant from the Science and Engineering
Research Board (SERB), Department of Science and Technology, Government of India.  M.V. was supported by a SERB National Science Chair, Government of India.
The authors have not declared any conflicts of interest.

There is no data associated with this manuscript.

\bibliographystyle{IEEEtran}
\bibliography{Comp-Sens}

\end{document}